%% file: paper.tex
\documentclass{article}

\usepackage{times}
\usepackage{graphicx} 
\usepackage{subfigure} 

\usepackage{natbib}

\usepackage{algorithm}
\usepackage{algorithmic}

\usepackage{amsmath,amssymb,amsthm}

\usepackage{hyperref}

\usepackage[accepted]{icml2017}

\icmltitlerunning{Sample Efficient Feature Selection for Factored MDPs}

\newtheorem{theorem}{Theorem}
\newtheorem{lemma}{Lemma}
\newtheorem{corollary}{Corollary}
\newtheorem{assumption}{Assumption}

\newcommand{\algnamelong}{Feature Selection Explore and Exploit}
\newcommand{\algnameshort}{FS-EE}

\title{Sample Efficient Feature Selection for Factored MDPs}
\author{Zhaohan Daniel Guo\footnote{Carnegie Mellon University, Pittsburgh, CA}, Emma Brunskill\footnote{Stanford University, Stanford, CA}}
\date{}

\begin{document}

\maketitle

\begin{abstract}
In reinforcement learning, the state of the real world is often represented by feature vectors. However, not all of the features may be pertinent for solving the current task. We propose \algnamelong~(\algnameshort), an algorithm that automatically selects the necessary features while learning a Factored Markov Decision Process, and prove that under mild assumptions, its sample complexity scales with the in-degree of the dynamics of just the necessary features, rather than the in-degree of all features. This can result in a much better sample complexity when the in-degree of the necessary features is smaller than the in-degree of all features.
\end{abstract}

\input{introduction}

\input{related}

\input{background}
\input{intuition}
\input{algorithm}

\input{theorysketch}

\input{experiments}
\input{conclusion}

\section*{Acknowledgments}

The research reported here was supported in part by ONR Young Investigator award, an NSF CAREER award and by the Institute of Education Sciences, U.S. Department of Education, through Grants R305A130215 and R305B150008 to Carnegie Mellon University. The opinions expressed are those of the authors and do not represent views of NSF, IES or the U.S. Dept. of Education.

\bibliography{paper}
\bibliographystyle{icml2017}

\newpage
\appendix
\twocolumn[
\centering
\Large
\vspace{0.1in}
\textbf{Appendix}
\vspace{0.5in}
]
\input{theory}

\input{toggle}

\end{document}

%% file: introduction.tex
\section{Introduction}
In many machine learning and AI control problems, choosing which features to represent the state of the domain is critical. Since the best representation is typically unknown, it is appealing to start with raw sensory input (like the pixels in a video game snapshot) or all possible features that might be relevant. Recent work in deep reinforcement learning \cite{mnih2013playing} has shown that it is possible to obtain great performance in some domains by using such 
representations. Unfortunately, the complexity of representing 
the dynamics, value function, and policies 
typically scales with the number of features, 
resulting in a large increase in the number of samples 
required to learn a good decision policy. While in some simulated domains this is not a critical limitation, in many high stakes domains 
(such as customer marketing, healthcare, education, and robotics) sample 
efficiency is very important. In many such RL settings, 
good performance still relies on using a small set of carefully hand-designed features. 
This process can be expensive, requiring domain experts to 
select the features, and may easily miss relevant features resulting in sub-optimal performance.
Much preferable are reinforcement learning algorithms 
whose sample complexity scales only with the number of relevant features needed to learn the optimal policy, and not on the total number of features defined.

In this paper we present theory that takes a step 
towards this goal, showing it is possible 
for an online RL algorithm, in the factored Markov Decision 
Processes setting, to achieve near-optimal average performance on all 
but a number of samples that scales with an important 
part of the model 
complexity of the necessary features, rather than the complexity of the whole feature set.  

In particular, we considered RL for Factored MDPs (FMDPs). FMDPs use feature vectors to represent states, enabling a compact encoding of real world domains. The sample complexity (the number of steps on which the algorithm may make non-near-optimal decisions) of RL algorithms for tabular MDPs scales at least linearly with the size of the state space \cite{strehl2009reinforcement}, which is exponential in the number of features if applied to FMDPs. Fortunately in an FMDP, the dynamics of each feature can depend on a small parent set of other features, so the sample complexity scales exponentially only with the size of the largest parent set (known as the in-degree) \cite{kearns1999efficient}. However, there exist 
many domains where some features' dynamics may be quite complex to model, 
but not be relevant to the underlying reward or value function. 
For example, when making tea, 
modeling the sky beyond the window may involve weather predictions, 
or knowledge of the date and time, 
but successful task completion may only rely on 
a sensor for water temperature. Video games often dedicate many pixels in order to show a pretty scene, but most of the time only the edges and outlines of objects actually matter to the gameplay. In such domains, the set of features necessary to learn the optimal value function may have a much smaller 
in-degree compared to all descriptive domain features. 

Our contribution is showing the existence of an RL algorithm whose sample complexity scales exponentially with only the in-degree (number of parents) of the necessary features. Our result is an exponential improvement over prior FMDL RL algorithms \cite{chakraborty2011structure} that do no feature selection, if the in-degree of the necessary features is smaller than the in-degree of the full feature set. Our algorithm does not assume knowledge of which nor how many features are necessary, nor any knowledge of parent sets. Even if the number of necessary features $M$ is given, a naive algorithm that attempts to figure out the set of necessary features by trying all $N$ (number of total features) choose $M$ subsets would yield a sample complexity of $O(N^M)$. $M$ can easily be larger than the in-degree of the necessary features, which would result in a substantially worse sample complexity.
   
Our key insight is to have the RL algorithm target specific goal states and leverage a failure to reach them to identify when the current guess for the in-degree is insufficient. Our approach assumes the domain has finite diameter, to ensure that we should be able to reach any state if we have good models, and also makes a mild assumption on the transition models that allows us to detect if the number of parents is insufficient without considering the full set of all features as parents. Our approach 
builds on work for RL in factored MDPs that does 
not require knowledge of the in-degree of an FMDP \cite{chakraborty2011structure} but goes significantly beyond this 
to tackle the feature selection problem 
during online learning. 

We focus on the theoretical improvement in sample complexity to show how to leverage feature selection in principle. Going forward we hope to leverage these insights towards practical FMDP RL algorithms with feature selection and guarantees of performance.

%% file: related.tex
\section{Related Work}

Prior work on factored MDP RL with formal theoretical 
bounds include 
Met-RMax \cite{diuk2009adaptive} and LSE-RMax \cite{chakraborty2011structure}, 
which does not require prior knowledge of the in-degree; however, 
such work does not perform feature selection. More 
recent work has significantly reduced the 
sample complexity of learning FMDPs \cite{hallak2015off}, 
but requires strong structural assumptions and only apply to the batch setting, which  does not account for the trade-off between exploration and exploitation.

The closest prior work that does feature selection for FMDP performs it as a post-processing step after solving the FMDP and uses the learned features for transfer learning \cite{kroon2009automatic}. In contrast, our algorithm learns the necessary features while doing online reinforcement learning. We also provide a formal theoretical analysis which is the first, to our knowledge, for this setting; other prior work focus more on practical algorithms without formal guarantees such as using multinomial regression with LASSO \citep{nguyen2013online}.

There does exist work for feature selection for 
value function estimation, but not for the FMDP setting; OMP-BRM/TD, and iFDD are algorithms that do feature selection in the setting where the value function is a linear combination of the features \cite{painter2012greedy,geramifard2011online}. While their performance is primarily dependent on the number of necessary features, 
these approaches depend on the assumption of a linear value function.


Feature selection can also be viewed as a form of model selection, where each model is a particular selection of features. Prior work has theoretical bounds for model selection such as the OAMS algorithm \cite{ortner2014selecting}; however those bounds depend on the square root of the number of models. For FMDPs the number of models grows doubly exponential with the number of features.

More generally for MDPs without a feature vector representation, the concept of feature selection translates to state abstraction -- ignoring features is equivalent to clustering all the states that match on the necessary features. An example is the U-Tree algorithm \cite{mccallum1996reinforcement}. While many state abstraction algorithms perform well empirically, they lack formal guarantees. 

%% file: background.tex
\section{Setting}

A finite FMDP is defined by a tuple $(S,A,P,R)$, where $S$ is a finite set of states, $A$ is a finite set of actions, $P$ is the transition distribution and $R$ is the reward distribution. Each state $s$ is a feature vector $(x_1, x_2,\dots, x_n)$ where $x_i \in Dom_i$ and $|Dom_i| = d$ \cite{kearns1999efficient}. The transition distribution factors over the state space i.e. $P(s_{t+1} | s_t, a_t ) = \prod_i P(x_{i, t+1} | s_t, a_t) = \prod_i P_{i}(s_{t+1} | Par_{i}(s_t), a_t)$. Each $P_{i}$ is the transition probability for feature $x_i$, dependent on its parent set of features $Par_{i}$. The notation $Par_{i}(s_t)$ denotes filtering the feature vector $s_t$ to only the features present in the set $Par_{i}$.
The reward is defined as $R(s,a) = \sum_j^{|R|} R_{j}(s,a)$, where each $R_{j}$ is an individual reward distribution for the $j$-th reward function.
$ P(R_{j}(s) | Par_{j}(s), a)$ is a discrete distribution with a domain of size $d$ just like a feature. Since features and rewards both utilize the same basic representation, the same approach 
can be used to learn feature transition dynamics and rewards. 
The in-degree of this FMDP is the size of the largest parent set over all features/rewards and actions i.e. $\max_{i} |Par_{i}|$.

Let $F'$ denote the set of all features. Given an FMDP, we assume there exists a set of necessary features $F$, such that if we ignored all features except the ones in $F$, we would get a smaller FMDP whose optimal value function is the same as the original FMDP.  This implies that the parents of features in $F$ are in $F$, and the parents of the rewards are also in $F$. 

We also assume the FMDP has a finite diameter $D$ \cite{ortner2007logarithmic}. A diameter $D$ means that for any two states $s_1,s_2$, the expected number of steps to go from $s_1$ to $s_2$ is at most $D$ under the best possible policy.

In this setting, the problem is to interact with an FMDP where the transition/reward dynamics are unknown (i.e. parent sets are unknown), the in-degree is unknown, and the necessary features are unknown, and execute a polity whose performance is $\epsilon$-close to the best possible policy. Interaction proceeds in steps, where in each step an algorithm takes an action, and observes the (stochastic) next state and reward. We measure performance of a policy $\pi$ using the average reward notion, where $U^\pi(s) = \lim_{T \rightarrow \infty} U^\pi(s, T)= \lim_{T \rightarrow \infty} \frac{1}{T}\mathbb{E}(\sum_{t=1}^T r_t | s_1 = s, \pi)$ \cite{kearns2002near}. Like prior FMDP work, we also assume $U^\pi(s)$ is independent of $s$ and can be denoted as just $U^\pi$ \cite{chakraborty2011structure}. Since we are working with finite samples, we assume the  $\epsilon$-return mixing time $T_{\epsilon}$ is given, same as in prior work \cite{chakraborty2011structure}. $T_{\epsilon}$ is such that for any policy $\pi$ and $T' \geq T_\epsilon$, $|U^\pi - U^\pi(T')| \leq \epsilon$ i.e. $T_\epsilon$ is long enough to see $\epsilon$-optimal average reward.

%% file: intuition.tex
\section{Difficulty of Feature Selection}

There is an inherent difficulty when trying to detect if you have the correct in-degree or not. Consider the following domain with $5$ binary features. Feature $f_i$ depends on all features $f_j$ where $j \leq i$ as well as on $f_5$. Feature $f_5$ acts as a toggle that toggles between easier transitions and harder transitions i.e. when $f_5=0$, the other features have a low probability of transitioning to a value of 1; when $f_5=1$, they have a much higher probability. Additionally, $f_i$ only has the probability of changing its value under a particular setting of its parent set values (e.g. $f_2$'s parent set is $(f_1,f_2,f_5)$ and only has a nonzero probability of changing to value 1 when $f_1=1$ and $f_2=0$). With this domain, the true in-degree is $5$, since $f_5$ depends on all the features. 

Suppose at some point our learning algorithm guesses that the in-degree for all features is $J=3$, and tries to use parent sets of size $3$ for all features. The reason for trying in-degrees less than $5$ is that the number of samples required to learn a model scales exponentially with the in-degree. If we started with guessing $J=5$, then the sample complexity of our algorithm would depend on the in-degree for all features, rather than just the necessary ones. Since the in-degree for $f_5$ is $5$, any parent sets of size $3$ learned would be incorrect. Moreover, it would be very difficult to even transition to any states where $f_5=1$, because it only has the possibility of transitioning to a value of 1 when its parent set values are a particular setting. Thus for most of the time we will stay in states where $f_5=0$. However, $f_5$ is also a parent of feature $f_1$. If we only observe states where $f_5=0$, then it is impossible for an algorithm to detect whether $f_5$ is a parent of $f_1$. This is because we would not get any data from $f_5=1$ so we don't know whether the transition dynamics of $f_1$ would be affected by different values of $f_5$. Therefore, only when we guess the in-degree is $J=5$ would we finally be able to reliably learn the parent set for $f_5$ and also detect that $f_5$ is a parent of $f_1$.

Now suppose we make a slight tweak to the example where $f_5$ is no longer a parent of $f_1$. We would still have the result where we would almost always observe states where $f_5=0$ if we guess an in-degree of $J<5$. Then similarly we would not be able to detect whether $f_5$ is a parent of $f_1$. Hence, we can never be sure whether a feature is a parent of another feature until we have guessed a high enough in-degree. This means that if we start trying to eliminate features such as $f_5$ early on, it would have to be temporary; once we gather more data and guess a higher in-degree, we would need to reassess whether we can detect any new dependencies between features and reinstate eliminated features as being necessary.

Thus, our algorithm will start out guessing small in-degrees and slowly increment what it thinks the in-degree is. It will only temporarily eliminate features based on what it can detect for the current in-degree. Once it increments the in-degree, it will recheck which features should be eliminated. Therefore, the benefit of feature selection is precisely in being able to achieve optimal performance as soon as the guess for the in-degree is correct for the necessary features. Then it should maintain optimal performance as the in-degree is incremented further. This leads to being able to achieve optimal performance faster and with a data dependence only on the in-degree of the necessary features.

Some work has been done for feature selection in the related setting of Sparse Linear Stochastic Bandits \cite{abbasi2012online}, where they show a lower bound with strong partial dependence on the total number features. Our results avoid a strong dependence on the total number of features and the in-degree of all features by making a mild assumption called the Superset Assumption (below) that we believe is often applicable in practice.

\subsection{Intuition for a Superset Test}

As the example above shows, it can very difficult to detect when the in-degree is too small for a feature, and in particular, difficult to detect for unnecessary features. Thus we may not be able to eliminate all unnecessary features, which can wind up as potential parents for necessary features, resulting in incorrect parent sets for necessary features. However this example is quite extreme, and in practice, it may be very possible to quickly gather enough data to detect when the in-degree is too small.

Suppose we have an incorrect parent set of size $3$ for an unnecessary feature $f$. If we looked at the parent set of all features, then clearly this would include the true parent set and thus would give an accurate transition/reward estimate for $f$. Then we would be able to detect that our parent set of size $3$ is incorrect since it would give a different transition/reward estimate. In practice, we may not need to compare to the parent set of all features; it may be enough to compare to a parent set that contains a few more true parent features than our incorrect parent set, since those additional parent features can result in a significantly different transition/reward estimate. We call this comparison process the Superset Test. We make an associated assumption called the Superset Assumption which says that looking at all supersets of double the size is enough to find a superset with a significantly different transition/reward estimate. By restricting the supersets to be at most double the size of the parent set we want to check, we can bound the number of samples required to perform this check.

Due to the pigeonhole principle, the data requirement for testing all supersets of some size $2J$ is only exponential in $2J$, which enables us to get a good bound. Note that at every step, because we observe all features of a state, we get data for all supersets of size $2J$. Thus we can define each superset of size $2J$ as a bucket, and every step we get a new sample for all buckets. In each bucket, we have $d^{2J}$ more buckets, one for each possible setting of values. We can apply the Pigeonhole principle to each of these inner buckets, which means we will need $O(d^{2J}m)$ steps to guarantee that some setting of values gets $m$ samples for every superset of size $2J$.

%% file: algorithm.tex
\section{Algorithm}

We first provide a basic overview of our algorithm, Algorithm \ref{alg:main}: \algnamelong~(\algnameshort), before discussing details. 

Algorithm \ref{alg:main} proceeds by fixing a possible in-degree $K$, starting with $K=1$ up to the total number of features. This is similar to LSE-RMax \cite{chakraborty2011structure}. For each $K$, our algorithm first performs exploration and feature selection (Algorithm \ref{alg:episode}) to identify the set of features whose dynamics can be modeled well with in-degree $K$. Then once it has selected those features, it calls a PAC RL algorithm for factored MDPs which then only considers the selected features and treats all other features as nonexistent.

We will prove in Section \ref{sec:theorysketch} that under the diameter assumption and the Superset Assumption, suitable instantiations of constants, and a suitable PAC RL algorithm for FMDPs, \algnameshort~obtains  
near optimal average reward on all but a number of steps that depends only on the (unknown) in-degree of the necessary features. Note that we do not need to know the number of necessary features in advance either. This translates to an exponential savings in samples if the in-degree of unnecessary features is larger than the in-degree of the necessary features; our algorithm will start acting near-optimally as soon as $K$ is at least as large as the in-degree of the necessary features. However, before $K$ is at least as large as the in-degree of the necessary features, there are no guarantees on performance and nothing can be determined about the features that are temporarily eliminated.

\begin{algorithm}[h]
	\begin{algorithmic}[1]
    \STATE \textbf{Input:} $m_1$, $m_2$, $H$
  \FOR{$K=1$ to \# of features}
   \STATE $F$=LearnAndSelect($K$,$m_1$,$H$)
   \STATE PAC--FactoredMDP--RL($K$, $F$,$m_2$)
 \ENDFOR
 \end{algorithmic}
 \caption{\algnameshort}
 \label{alg:main}
\end{algorithm}

\begin{algorithm}[h]
	\begin{algorithmic}[1]
 \STATE \textbf{Input:} $K$, $m$, $H$
\STATE Set $G$ as all possible $F$ choose $2K$ feature-value vectors
   \WHILE{Exists a element $g$ of $G$ not visited $m$ times} \label{lin:while}
   \STATE Create MDP $M$ where reward for states matching $g$ is $R_{max}$, reward for all other states is 0 \label{lin:targetmdp}
   \STATE Compute optimistic policy $\pi_o$ for $M$
   \STATE stuck=True
	   \FOR{$t=1$ to $H$} \label{lin:befor}
	     \STATE Run $\pi_o$
         \STATE Use Adaptive $k$-Meteorologist Algorithm to update possible parent sets for each feature \label{lin:adak}
         \STATE If a feature-value vector that has not yet been visited $m$ times is visited, set stuck=False and break out of loop
	   \ENDFOR \label{lin:endfor}
	  \IF{stuck}
	  	\STATE Superset Test: Eliminate possible parent sets that predict significantly different dynamics than their supersets \label{lin:ss}
        \STATE Shrink $F$ to be remaining features $F$ that still have possible parent sets \label{lin:elim}
        \STATE Shrink $G$ to be all remaining $F$ choose $2K$ feature-value vectors
	  \ENDIF
   \ENDWHILE \label{lin:endwhile}
   \STATE \textbf{Return} remaining $F$
   \end{algorithmic}
 \caption{LearnAndSelect}
 \label{alg:episode}
\end{algorithm}

We now describe our algorithm in further detail. The purpose of Algorithm \ref{alg:episode} is twofold. One, it explores to gather $m$ samples for all potential parent sets to learn an accurate dynamics model. Two, under the diameter assumption and Superset Assumption, it detects features whose transition/reward dynamics cannot be accurately modeled using a parent set of size $K$ and eliminates them.

To achieve both goals, Algorithm \ref{alg:episode} repeatedly picks a target feature-value vector of size $2K$ to visit (line \ref{lin:while}--\ref{lin:endwhile}). A feature-value vector is a particular instantiation of values for a set of features. We will describe Algorithm \ref{alg:episode} through a concrete example. Consider an FMDP with 3 binary features $(f_1, f_2, f_3)$. Suppose $K=1$. Then all feature-value vectors of size $2$ for $(f_1, f_3)$ are simply $(f_1=0,f_3=0),(f_1=0,f_3=1),(f_1=1,f_3=0),(f_1=1,f_3=1)$. Then the first thing Algorithm \ref{alg:episode} does is form the set $G$ of all possible feature-value vectors for all subsets of 2 features: $(f_1,f_2),(f_1,f_3),(f_2,f_3)$. Then the algorithm attempts to target exploration to one particular feature-value vector. Let $g = (f_1=0,f_2=1)$ be the first target. To reach $(f_1=0,f_2=1)$, Algorithm \ref{alg:episode} creates an FMDP where the reward function for any state that matches $g$ is $R_{\max}$ i.e. states $(0,1,0),(0,1,1)$ (line \ref{lin:targetmdp}). The reward for all other states is 0. Then an optimistic policy $\pi_o$ is computed and followed for up to $H$ steps to try to visit states that match the target feature-value vector (line \ref{lin:befor}--\ref{lin:endfor}).

The optimistic policy is computed by choosing a parent set for each feature out of the possible parent sets that results in an FMDP with the largest optimal value. We also use the Adaptive $k$-Meteorologists Algorithm \cite{diuk2009adaptive} to update potential parent sets and the dynamics model from the gathered data during exploration (line \ref{lin:adak}); see section \ref{subsec:meteor} for an overview.

Following the optimistic policy $\pi_o$ will then result in 3 possible outcomes (under the diameter assumption): 1) it will reach the target feature-value vector, 2) it will visit some other feature-value vector that has not yet been visited $m$ times, or 3) it will get stuck visiting already visited feature-value vectors. This is similar to the PAC-EXPLORE algorithm \cite{guo2015concurrent}. The first two outcomes are both good and end up collecting data towards feature-value vectors that have not yet been visited $m$ times. The third outcome implies that one of the parent sets used for computing $\pi_o$ is incorrect. This is because the diameter assumption guarantees that any state is reachable on average in $D$ (diameter) steps. If all of the parent sets used in computing the optimistic policy $\pi_o$ were correct, then $\pi_o$ would be expected to reach the target feature-value vector within $O(D^2)$ steps. Thus we can detect that we have not reached the target in the expected time nor accumulated any new samples for less visited feature-value vectors, and determine that we are stuck.

If we are stuck, then we perform the Superset Test (line \ref{lin:ss}) to figure out which parent sets are incorrect and eliminate them. The Superset Test will compare the predictions of all potential parent sets of all features with the predictions of all of their possible supersets of size $2K$. This leverages the idea that if a parent set is the correct parent set for a feature, then any superset of that parent set will give the same prediction; thus if a superset gives a different prediction then that parent set must be missing necessary features. The Superset Assumption guarantees that if we are stuck, then the Superset Test will always find at least one incorrect parent set to eliminate.

Suppose we have explored for some time and our current target is now $g = (f_2=1,f_3=1)$. Suppose we get stuck. Then after $H$ steps, we would detect that we got stuck and then perform the Superset Test. Suppose we first look at feature $f_2$ and the possible parents remaining are $f_2$ and $f_3$ i.e. $f_1$ has been eliminated. The Superset Test will test the predictions of each parent set with all of their possible supersets. For the parent set $(f_2 = 0)$, the supersets of size 2 are $(f_2 = 0, f_3 = 0)$ and $(f_2 = 0, f_3 = 1)$. Let $\hat{P}(f_2 = 0 | f_2 = 0)$ be the estimated transition for $f_2$ with parent set $(f_2 = 0)$. Let $\hat{P}(f_2 = 0 | f_2 = 0, f_3 = 0)$ be the estimated transition for $f_2$ with the parent set $(f_2 = 0, f_3 = 0)$. The Superset Test will check whether $|\hat{P}(f_2 = 0 | f_2 = 0, f_2 = 0) -  \hat{P}(f_2 = 0 | f_2 = 0)|$ is above some threshold. If it is, then it would mean that $f_2$ is not a parent of $f_2$ and so both parent sets $(f_2 = 0)$ and $(f_2 = 1)$ would be eliminated. The Superset Test then continues with the other possible parents and their supersets, and then will check the other features in the same way. If all possible parent sets for a feature $f$ has been eliminated, then $f$ itself is eliminated (line \ref{lin:elim}).

Each time Algorithm \ref{alg:episode} targets a feature-value vector for exploration, either some target that has not yet been visited $m$ times will get visited, or the Superset Test will eliminate a potential parent for a feature. Thus eventually the algorithm will terminate. Once Algorithm \ref{alg:episode} terminates, we are left with the remaining features $F$ which we know that our model is capable of reaching and exploring. Thus Algorithm \ref{alg:main} calls a generic PAC RL algorithm for FMDPs with in-degree $K$ with remaining features $F$, ignoring the eliminated features completely.

\subsection{Adaptive $k$-Meteorologist}
\label{subsec:meteor}

The Adaptive $k$-Meteorologists Algorithm is used to update the predictions of parent sets as well as to eliminate incorrect parent sets as more and more data is accumulated \cite{diuk2009adaptive}. It maintains MLE estimates for transition and reward functions for every possible parent set of every feature. It also keeps track of the mean squared error (MSE) of those MLE estimates for every parent set. Once two different parent sets for a feature $f$ have an accurate enough MSE estimate but different predictions, the algorithm will eliminate the parent set with the higher MSE. Thus eventually with enough data, all the parent sets that remain for a feature $f$ will reach consensus in their predictions. Note that because this algorithm only performs pairwise comparisons of parent sets, it will always leave at least one parent set for every feature, thus this cannot be used to eliminate features as it cannot detect when all parent sets of some size $J$ are incorrect.

%% file: theorysketch.tex
\section{Theoretical Analysis}
\label{sec:theorysketch}

This section presents the main theorem as well as the supporting lemmas for the performance of \algnameshort. In the first section (Section \ref{sub:ass}), we present the assumptions we make. In Section \ref{sub:main} we present the main theorem and proof, which relies on two large lemmas. In Section \ref{sub:small}, we present several small lemmas that are used in the large lemmas. In Section \ref{sub:adak} we review the theory behind the Adaptive $k$-Meteorologist algorithm, which is used in the large lemmas. In Section \ref{sub:episode}, we build up to the first large lemma used in the main theorem. In Section \ref{sub:exploit} we present the other large lemma used in the main theorem.

Detailed proofs are in the appendix.

\subsection{Assumptions}
\label{sub:ass}

We first present the two main assumptions that we make.

\begin{assumption}
\label{ass:diameter}
(Diameter Assumption)
We assume the FMDP has diameter $D$. A diameter $D$ is defined as follows. Let $s,s'$ be any two states. Let $D_\pi(s,s')$ be the random variable for the number of steps it takes policy $\pi$ to start at $s$ and reach $s'$ the first time. Let $D(s,s') = \min_{\pi}\mathbb{E}(D_{\pi}(s,s'))$. A diameter $D$ means that $D \geq \max_{s,s'}D(s,s')$. It is an upper bound on the expected number of steps it takes to go between any two states.
\end{assumption}

\begin{assumption}
	\label{ass:superset}
    (Superset Assumption)
	Let $W$ be the wrong parent set for feature $f_i$. Let $U$ be the true parent set. Suppose new data has the probability of visiting a state where $W$ gives an $O(\epsilon)$-incorrect estimate of the transition probability being at least $O(\epsilon)$. Then there exists a superset $W'$ of $W$ where $W' \subset (W \cup U)$ and $|W'| \leq 2K$ such that for all settings of values $W' - W$, the transition estimate using $W$ differs more than $O(\epsilon)$ from the transition estimate using $W'$ (that is estimated from new data).
\end{assumption}

\subsection{Main Theorem}
\label{sub:main}

In this section, we give the main theorem and its proof. Note that for \algnameshort~we set $m_1 = m_2 = m$ and the precise value for $m$ is determined later (eqn \ref{eqn:m}).

\begin{theorem}
	\label{thm:main}
	Given $\epsilon>0, 0.5 > \delta > 0$. Let $T_{\epsilon,D} =\max(D,T_\epsilon)$. Let $J$ be the in-degree of the necessary features. Recall $n$ is the total number of features. Then the following is true of \algnameshort~with probability $1-\delta$
	\begin{enumerate}
		\item The total number of steps taken up to $K=J$ is $O\left(\frac{J^4 T_{\epsilon,D}^{18} n^{4J+c} |A|^4 d^{8J+c} R_{\max}^4}{\epsilon^6}\log\left(\frac{n|A|dR_{\max}KT_{\epsilon,D}}{\epsilon\delta}\right)\right)$ for some constant $c$
		\item For all $K \geq J$ i.e. at least as large as the in-degree of the necessary features, the average reward is $\epsilon$-optimal i.e. $|U - U^*| \leq \epsilon$
	\end{enumerate}
\end{theorem}

\begin{proof} (Sketch)
This follows from putting together the two large lemmas Lemma \ref{lem:selection} and Lemma \ref{lem:rlwithselected} to count the number of steps taken for each $K$. Summing over $K$ up to $K=J$ and plugging in $m$ results in the stated bound. For $K \geq J$,  $\epsilon$-optimality follows Lemma \ref{lem:rlwithselected}.
\end{proof}

\subsection{Small Lemmas}
\label{sub:small}

\begin{lemma}
	\label{lem:necessary}
    (Necessary Feature Lemma)
	For any policy $\pi$
	\begin{align}
	Q^{\pi}(s, a, T) = Q^{\pi}(z, a, T)
	\end{align}
	where $z$ is the state $s$ restricted to only the necessary features from $F$ i.e. $Q$-functions only depend on the necessary features. Furthermore, the transition dynamics of the unnecessary features have no effect on $Q$-functions.
\end{lemma}

\begin{proof}(Sketch)
This follows by induction on value iteration, and the observation that reward functions and necessary features only depend on necessary features for their dynamics.
\end{proof}

\begin{lemma}
	\label{lem:simulation}
    (Simulation Lemma) \cite{kearns1999efficient}
	Let $M$ be an FMDP over $n$ state variables with $l$ CPT entries in the transition model. Let $M'$ be an approximation to $M$ where all the CPTs differ by at most $\alpha = O((\epsilon/T^2 l R_{\max})^2)$. Then for any policy $\pi$, $|U^\pi_M(T) - U^\pi_{M'}(T)| \leq \epsilon$. Subbing in $l$ and noting that $\alpha$ error in rewards is already covered results in $\alpha = O((\epsilon/(T^2 nd^{K}|A| R_{\max}))^2)$.
\end{lemma}

\begin{lemma}
	\label{lem:ee}
    (Explore or Exploit Lemma)
	Fix a policy $\pi$. Let $M$ and $M_K$ be MDPs such that $M$ and $M_K$ agree on some states, but differ in dynamics and rewards for other states. Then $|U^{\pi}_{M_K}(T) - U^{\pi}_{M}(T)| \leq T R_{\max}P(escape)$ where $P(escape)$ is the probability of visiting a state in which the two models differ.
\end{lemma}

\begin{proof}(Sketch)
The key observation is that trajectories that do not escape are identical for both $M$ and $M_K$, thus the probability of escape is the same.
\end{proof}

\begin{corollary}
	Suppose $\pi_1$ is the optimal policy for $M_k$ and $\pi_2$ is the optimal policy for $M$. Suppose $U^*_{M_K}(T) \geq U^*_{M}(T)$ i.e. $M_K$ is optimistic. Then $U^{\pi_1}_{M_K} \geq U^{\pi_2}_{M} - T R_{\max}P(escape)$.
\end{corollary}

\subsection{Adaptive $k$-Meteorologist Algorithm}
\label{sub:adak}

This is from the Adaptive $k$-Meteorologist Algorithm \cite{diuk2009adaptive}. Similar to Met-RMax \cite{diuk2009adaptive}, each sub-algorithm of our algorithm is a candidate parent set of features of size $K$. Thus there are $k = \binom{n}{K}$ sub-algorithms. By Hoeffdings, we need $O(\frac{d}{\epsilon_1^2}\log(d/\delta_1))$ samples to learn the discrete distribution of a parent set to an $L_1$ accuracy of $\epsilon_1$ (we apply Hoeffdings $d$ times, learning the probability of each outcome with Hoeffdings). There are $d^K$ possible sets of values for a parent set of size $K$ so the sample complexity of each sub-algorithm is $O\left(d^K \frac{d}{\epsilon_1^2}\log(d/\delta_1)\right)$. Then the sample complexity of the Adaptive $k$-Meteorologist algorithm can be simplified to $O\left(\frac{n^K d^{K+1} K }{\epsilon_1^2} \log \frac{n}{\delta}\right)$

\subsection{LearnAndSelect}
\label{sub:episode}

\begin{lemma}
	\label{lem:episode}
    (Exploration Episode Lemma)
	The following holds w.p. $1-\delta_1$. At the end of each iteration of the while loop (line \ref{lin:while} -- line \ref{lin:endwhile}) in the LearnAndSelect algorithm (Algorithm \ref{alg:episode}), one of two things will happen: either the target $g$ or another feature-value vector that has not been visited $m$ times will be visited, or some feature will be eliminated as a possible parent for some other feature.
\end{lemma}

\begin{proof}(Sketch)
The key is Lemma \ref{lem:ee} and the diameter assumption. The diameter assumption tells us that it is possible to reach the goal in expected $D$ steps. By the Markov Inequality, the probability of reaching the goal within $2D$ steps is at least $\frac{1}{2}$ with some policy. Thus if we have a good policy, then we simply need to keep running it many times and we can reach the goal with high probability.

In the other case of Lemma \ref{lem:ee} where the probability of escape is significant, we can accumulate new data for another feature-value vector, or new data to run the Superset Test. This data is also used towards Adaptive $k$-Meteorologists to help with eliminating incorrect parent sets. For the Superset test, we apply the Superset assumption (Assumption \ref{ass:superset}). Then we know that whenever we run the superset test, we will eliminate at least one parent set. To get enough data, we need
\begin{align}
H = O\left(\frac{d^{6K+1}n^{4} K^2 |A|^4 R_{\max}^4}{\epsilon^5 \max(D,T_\epsilon)^{-18}}\log\left(\frac{nm}{\delta_1}\right)\right) \label{eqn:supersetdata}
\end{align}

The value of $m$ comes from the sample complexity result of learning a single parent set from Adaptive $k$-Meteorologist, and from Lemma \ref{lem:simulation}:
\begin{align}
m=O\left(\frac{K^2 n^4|A|^4d^{4K+1} R_{\max}^4}{\epsilon^4 \max(D,T_\epsilon)^{-16}}\log(nd/\delta_1)\right) \label{eqn:m}
\end{align}
\end{proof}

\begin{lemma}
	\label{lem:selection}
    (LearnAndSelect Lemma)
	The following holds w.p. $1-\delta_1$. After LearnAndSelect (Algorithm \ref{alg:episode}) is finished, all targets $g$ will either have been visited $m$ times, or one of its features will have been eliminated. If $K \geq J$, all necessary features will remain. This will take $O\left(\frac{K^2 n^{3K+4}|A|^4d^{7K+1} R_{\max}^4}{\epsilon^5 \max(D,T_\epsilon)^{-18} }\log(nm/\delta_1)\right)$ steps.
\end{lemma}

\begin{proof}(Sketch)
The first part follows from Lemma \ref{lem:episode}. To count steps, note every time the while loop is run, either some new feature-value vector is visited, or we perform a Superset Test. For superset tests, there are $O((dn)^{K})$ tests (each test eliminates at least one parent set), so there are $O(H(dn)^{K})$ steps. For visiting feature-value vectors, all $O((dn)^{3K})$ (an upper bound on the number of possible feature-value vectors of size $2K$) targets need to be visited $m$ times so $O((dn)^{3K}m)$ visits are needed. Thus, the total count combines the count of those two cases.
\end{proof}

\subsection{Instantiating PAC-FactoredMDP-RL}
\label{sub:exploit}

\begin{lemma}
	\label{lem:rlwithselected}
    (PAC-FactoredMDP-RL Lemma)
	Suppose $K \geq J$. Then we instantiate PAC-FactoredMDP-RL with a specific algorithm that is a simple variant of Adaptive $k$-Meteorologist (see appendix). Then it will execute for $O\left(\frac{K^2  n^{3K+4}|A|^4d^{7K+1} R_{\max}^4}{\epsilon^5 \max(D,T_\epsilon)^{-18}} \log(nm/\delta_1)\right)$ steps. The average reward over LearnAndSelect (Algorithm \ref{alg:episode}) and PAC-FactoredMDP-RL combined is $\epsilon$-optimal, with probability $1-\delta$.
\end{lemma}

\begin{proof} (Sketch)
We know that PAC-FactoredMDP-RL will achieve near-optimal performance, so all we need is to run it long enough to offset the sub-optimal LearnAndSelect i.e. so that the total reward averaged over the running time of both is $\epsilon$-close.
\end{proof}

%% file: experiments.tex
\section{Experiments}

We conducted a small experiment to show the possibility of practical improvement through feature selection. Similar to prior work \cite{chakraborty2011structure}, we did not run the additional targeted exploration to eliminate features of our approach. However, we do a passive form of feature selection by continuing to evaluate whether we can eliminate features by running a background superset test on the data collected so far: note this does not involve any explicit data collection, but can eliminate features. 
When we increment the in-degree and some of the transition and reward functions have not yet accumulated enough data for the new in-degree, we fall back on the previous in-degree's model.

To illustrate our ideas, we consider a small new toy domain, Toggle.  Toggle has one binary necessary feature $f_1$ and potentially two binary unnecessary features $f_2,f_3$. There are two actions, $a_1,a_2$. The parent set of $f_1$ is $(f_1)$, and the parent sets of $f_2,f_3$ are both $(f_2,f_3)$; this means the in-degree of necessary features is $1$ and the in-degree of the whole domain is $2$. We use $m=100$. The two unnecessary features $f_2,f_3$ satisfy the Superset Assumption. Full domain details are provided in appendix B.


\begin{figure}[h]
	\includegraphics[scale=0.4]{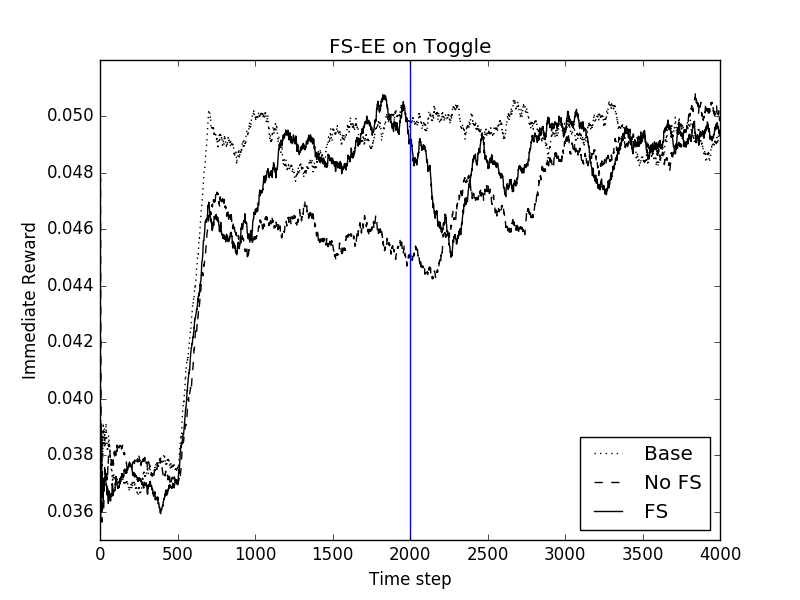}
	\caption{Feature selection (FS) can eliminate unnecessary features early for good performance, and when the in-degree is incremented at step 2000, suffers only a little drawback before becoming optimal.}
    \label{fig:immediate}
\end{figure}

We roughly optimized the number of steps before the in-degree is incremented to 2000 for best performance on the domain without feature selection. We also ran FS-EE on the base domain which only has the necessary feature as a measure of the best possible performance.
 
Introducing unnecessary features slows down learning, but when it is possible to start eliminating features, it can cut short the time needed to learn the unnecessary features and more quickly converge to a good policy. Figure \ref{fig:immediate} is the graph of their immediate rewards, and shows this behavior where feature selection improves performance for the first 2000 steps over no feature selection. This is due to quickly eliminating the unnecessary features and quickly learning the optimal policy. Then at step 2000 when the in-degree is incremented to 2, it suffers a little before climbing back up and matching the performance of the others. This dip in performance is due to the eliminated features becoming necessary again and thus a little more data is needed to learn them. This behavior matches what we expect -- the benefit of feature selection is due to near-optimal performance as soon as the in-degree is correct for the necessary features. Later on when the in-degree is incremented to 3, feature selection continues to match the performance of the others, with no dip in performance. 

Toggle illustrates the benefit of our approach in a simple setting. We also compared our approach to prior Factored PAC MDP RL algorithms LSE-RMAX\cite{chakraborty2011structure} and MET-RMAX\cite{diuk2009adaptive} on a standard FMDP domain, Stock Trading. This is mostly a sanity check, as all variables are required to represent the optimal values in this domain. Indeed, our approach is comparable to prior work: our algorithm FS-EE achieves a cumulative reward of approximately 5500 which is greater than MET-RMAX and near LSE-RMAX.


%% file: conclusion.tex
\section{Discussion and Conclusion}
We proposed the algorithm \algnameshort~for performing feature selection while solving FMDPs. We showed it has a sample complexity dependence that scales as an exponential function of the in-degree of the necessary feature set, potentially an exponential improvement over the in-degree of the 
full feature set. We illustrated that if feature selection is not needed, our approach is comparable to prior PAC RL algorithms for factored MDPs, but if feature selection is needed, our approach can lead to significantly improved performance on a toy domain. 

Many interesting questions remain. In the near term, it would be nice to investigate if it is possible to stop incrementing the cardinality of the parent sets in the dynamics model of necessary features, $K$ after a certain point.  This would halt any need for further exploration if larger parent sets are required.  We have discussed this possibility and presented an example where it seems extremely difficult. Perhaps a more tangible question is whether additional mild assumptions can be made to make this possible.
Another second short term question is whether we can further relax the Superset assumption, though it seems likely to apply in most practical problems. 

Over the longer term, one key issue is  the importance of directed feature selection. Just as many domains do not require directed exploration (such as the recent results on Atari domains), it may frequently be possible to easily learn the dynamic Bayesian network parent structure of all features passively, without requiring explicit directed exploration. In such situations our approach of doing directed feature exploration will be an unnecessary overhead. Yet it seems likely that 
in other settings directed feature selection could be very helpful, particularly in situations (such 
as many in supervised learning) where only a tiny subset of the potential feature set is required. 

Finally, the success of deep learning reinforcement learning has created approaches that 
implicitly construct features of the environment. Yet an open question is whether the 
data required to learn with such techniques scales with the input feature representation, 
or the input network structure, or rather only as function of the underlying learned 
features (and perhaps minimal network structure needed to represent such features). 
This is an important and interesting question both empirically and theoretically.

%% file: theory.tex
\section{Theoretical Analysis}

This section presents the main theorem as well as the supporting lemmas for the performance of \algnameshort. In the first section (Section \ref{sub:ass}), we present the assumptions we make. In Section \ref{sub:main} we present the main theorem and proof, which relies on two large lemmas. In Section \ref{sub:small}, we present several small lemmas that are used in the large lemmas. In Section \ref{sub:adak} we review the theory behind the Adaptive $k$-Meteorologist algorithm, which is used in the large Lemmas. In Section \ref{sub:episode}, we build up to the first large lemma used in the main theorem. In Section \ref{sub:exploit} we present the other large lemma used in the main theorem.

\subsection{Assumptions}
\label{sub:ass}

We first present the two main assumptions that we make.

\begin{assumption}
\label{ass:diameter}
(Diameter Assumption)
We assume the FMDP has diameter $D$. A diameter $D$ is defined as follows. Let $s,s'$ be any two states. Let $D_\pi(s,s')$ be the random variable for the number of steps it takes policy $\pi$ to start at $s$ and reach $s'$ the first time. Let $D(s,s') = \min_{\pi}\mathbb{E}(D_{\pi}(s,s'))$. A diameter $D$ means that $D \geq \max_{s,s'}D(s,s')$. It is an upper bound on the expected number of steps it takes to go between any two states.
\end{assumption}

\begin{assumption}
	\label{ass:superset}
    (Superset Assumption)
	Let $W$ be the wrong parent set for feature $x_i$. Let $U$ be the true parent set. Suppose new data has the probability of visiting a state where $W$ gives an $O(\epsilon)$-incorrect estimate of the transition probability being at least $O(\epsilon)$. Then there exists a superset $W'$ of $W$ where $W' \subset (W \cup U)$ and $|W'| \leq 2K$ such that for all settings of values $W' - W$, the transition estimate using $W$ differs more than $O(\epsilon)$ from the transition estimate using $W'$ (that is estimated from new data).
\end{assumption}

\subsection{Main Theorem}
\label{sub:main}

In this section, we give the main theorem and its proof. Note that for \algnameshort~we set $m_1 = m_2 = m$ and the precise value for $m$ is determined later (eqn \ref{eqn:m}).

\begin{theorem}
	\label{thm:main}
	Given $\epsilon>0, 0.5 > \delta > 0$. Let $T_{\epsilon,D} =\max(D,T_\epsilon)$. Let $J$ be the in-degree of the necessary features. Recall $n$ is the total number of features. Then the following is true of \algnameshort~with probability $1-\delta$
	\begin{enumerate}
		\item The total number of steps taken up to $K=J$ is $O\left(\frac{J^4 T_{\epsilon,D}^{18} n^{4J+c} |A|^4 d^{8J+c} R_{\max}^4}{\epsilon^6}\log\left(\frac{n|A|dR_{\max}KT_{\epsilon,D}}{\epsilon\delta}\right)\right)$ for some constant $c$
		\item For all $K \geq J$ i.e. at least as large as the in-degree of the necessary features, the average reward is $\epsilon$-optimal i.e. $|U - U^*| \leq \epsilon$
	\end{enumerate}
\end{theorem}

\begin{proof}
Note that our algorithm increments $K$ up to $n$. Thus using a union bound to bound the error for each $K$, we would need an error tolerance of $\delta/n$ for each $K$.

Putting together the LearnAndSelect lemma (Lemma \ref{lem:selection}), and the PAC-FactoredMDP-RL lemma (Lemma \ref{lem:rlwithselected}), the number of steps of PAC-FactoredMDP-RL dominates, resulting in $O\left(\frac{K^3 \max(D,T_\epsilon)^{18} n^{3K+4}|A|^4d^{7K+1} R_{\max}^4}{\epsilon^6}\log(nm/\delta)\right)$ steps.

Then just summing over $K$ up to $K=J$ results in $O\left(\frac{J^4 \max(D,T_\epsilon)^{18} n^{4J+c} |A|^4 d^{8J+c} R_{\max}^4}{\epsilon^6}\log(nm/\delta)\right)$, where $c$ is some constant.

Now suppose $K \geq J$. Being $\epsilon$-optimal follows from the PAC-FactoredMDP-RL Lemma.

Finally plugging in $m$ (equation \ref{eqn:m}) gets the final bound.
\end{proof}

\subsection{Small Lemmas}
\label{sub:small}

\begin{lemma}
	\label{lem:necessary}
    (Necessary Feature Lemma)
	For any policy $\pi$
	\begin{align}
	Q^{\pi}(s, a, T) = Q^{\pi}(z, a, T)
	\end{align}
	where $z$ is the state $s$ restricted to only the necessary features from $F$ i.e. $Q$-functions only depend on the necessary features. Furthermore, the transition dynamics of the unnecessary features have no effect on $Q$-functions.
\end{lemma}

\begin{proof}

Let $\pi$ be given. Let $s = (z, y)$ be the state decomposed into necessary features $z$ and unnecessary features $y$. Initialize $Q(\cdot, \cdot, 0)$ to $0$. We will perform induction. The base case for $Q(\cdot, \cdot, 0)$ is trivially true since it is a constant.

By induction
\begin{align}
& Q(s,a,T+1) \\
&= R(s,a) + \sum_{s'}P(s'|s,a) \max_{a'}Q(s',a',T) \\
&= \sum_i R_{i,a}(s) \\
&+ \sum_{s'}\prod_i P_{i,a}(s' | Par_{i,a}(s)) \max_{a'}Q(s',a',T) \\
&= \sum_i R_{i,a}(z) \\
&+ \sum_{s'}\prod_i P_{i,a}(s' | Par_{i,a}(s)) \max_{a'}Q(z',a',T) \\
&= \sum_i R_{i,a}(z) \\
&+ (\sum_{z'}\prod_{i \in z'} P_{i,a}(s' | Par_{i,a}(s)) \max_{a'}Q(z',a',T) \\
&\cdot \sum_{y'}\prod_{i \in y'} P_{i,a}(s' | Par_{i,a}(s))) \\
&= \sum_i R_{i,a}(z) \\
&+ \sum_{z'}\prod_{i \in z'} P_{i,a}(s' | Par_{i,a}(s)) \max_{a'}Q(z',a',T) \\
&= Q(z', a', T) 
\end{align}
Note that the dynamics of the unnecessary features make no difference.

\end{proof}

\begin{lemma}
	\label{lem:simulation}
    (Simulation Lemma) \cite{kearns1999efficient}
	Let $M$ be an FMDP over $n$ state variables with $l$ CPT entries in the transition model. Let $M'$ be an approximation to $M$ where all the CPTs differ by at most $\alpha = O((\epsilon/T^2 l R_{\max})^2)$. Then for any policy $\pi$, $|U^\pi_M(T) - U^\pi_{M'}(T)| \leq \epsilon$.
	
	With our notation, $l = O(n|A|d^K)$, since there is one parent set for every feature and action, and $d^K$ possible settings for each parent set of size $K$. Furthermore, we also have additional error in the reward function. An error of $\alpha$ for the reward translates to an error of $\alpha T$ for $U^\pi_{M'}$. Since $\alpha T \leq \epsilon$, it is enough for $\alpha = O((\epsilon/T^2 nd^{K}|A| R_{\max})^2)$ to also cover error in rewards.
\end{lemma}

\begin{lemma}
	\label{lem:ee}
    (Explore or Exploit Lemma)
	Fix a policy $\pi$. Let $M$ and $M_K$ be MDPs such that $M$ and $M_K$ agree on some states, but differ in dynamics and rewards for other states. Then $|U^{\pi}_{M_K}(T) - U^{\pi}_{M}(T)| \leq T R_{\max}P(escape)$ where $P(escape)$ is the probability of visiting a state in which the two models differ.
\end{lemma}

\begin{proof}

Let $\tau_1$ denote trajectories that stay within states where the two models agree and $\tau_2$ denote trajectories where there are escapes to other states. Then
\begin{align}
& |U^{\pi}_{M_K}(T) - U^{\pi}_{M}(T)| \\
&= \frac{1}{T}|\sum_{\tau,|\tau| = T}P_{M}(\tau)R(\tau) - \sum_{\tau,|\tau| = T}P_{M_K}(\tau)R(\tau)| \\
&\leq \frac{1}{T}|\sum_{\tau_1}P_{M}(\tau_1)R(\tau_1) - \sum_{\tau_1}P_{M_K}(\tau_1)R(\tau_1)| \\
&+ \frac{1}{T}|\sum_{\tau_2}P_{M}(\tau_2)R(\tau_2) - \sum_{\tau_2}P_{M_K}(\tau_2)R(\tau_2)| \\
&\leq \frac{1}{T}|\sum_{\tau_2}P_{M}(\tau_2)R(\tau_2) - \sum_{\tau_2}P_{M_K}(\tau_2)R(\tau_2)| \\
&\leq \frac{1}{T}\sum_{\tau_2}|P_{M}(\tau_2)R(\tau_2) - P_{M_K}(\tau_2)R(\tau_2)| \\
&\leq \frac{1}{T} T R_{\max}\sum_{\tau_2}|P_{M}(\tau_2) - P_{M_K}(\tau_2)| \\
&= R_{\max}P(escape)
\end{align}
Because non-escapes result in exactly the same trajectories with the same dynamics, so the probability of escaping to the other states is the same in both $M$ and $M_K$.
\end{proof}

\begin{corollary}
	Suppose $\pi_1$ is the optimal policy for $M_k$ and $\pi_2$ is the optimal policy for $M$. Suppose $U^*_{M_K}(T) \geq U^*_{M}(T)$ i.e. $M_K$ is optimistic. Then $U^{\pi_1}_{M_K} \geq U^{\pi_2}_{M} - T R_{\max}P(escape)$.
\end{corollary}

\begin{proof}
\begin{align}
U^{\pi_1}_{M} &\geq U^{\pi_1}_{M_K} - R_{\max}P(escape) \\
&\geq U^{\pi_2}_{M} - R_{\max}P(escape)
\end{align}
\end{proof}

\subsection{Adaptive $k$-Meteorologist Algorithm}
\label{sub:adak}

This is from the Adaptive $k$-Meteorologist Algorithm \cite{diuk2009adaptive}. Suppose there are $k$ sub-algorithms (i.e. potential parent sets). Then the sample complexity of the Adaptive $k$-Meteorologist algorithm is
$ O(\frac{k}{\epsilon^2} \log \frac{k}{\delta}) + \sum_{i=1}^k \zeta_i\left(\frac{\epsilon}{8}, \frac{\delta}{k+1} \right)$,
where $\zeta_i$ is the sample complexity of a sub-algorithm.

Similar to Met-RMax \cite{diuk2009adaptive}, each sub-algorithm of our algorithm is a candidate parent set of features of size $K$. Thus there are $k = {n \choose K}$ sub-algorithms. Each sub-algorithm uses the samples as counts for an MLE estimate of the transition/reward multinomial distribution.
Then by Hoeffdings, we need $O(\frac{d}{\epsilon_1^2}\log(d/\delta_1))$ samples to learn each set of values of a parent set to an $L_1$ accuracy of $\epsilon_1$ (we apply Hoeffdings $d$ times, learning the probability of each outcome with Hoeffdings). Then there are $d^K$ possible sets of values for a parent set of size $K$. Then the sample complexity of each sub-algorithm is 
\begin{align}
O\left(d^K \frac{d}{\epsilon_1^2}\log(d/\delta_1)\right) \label{eqn:subalg}
\end{align}

 Then the sample complexity of the Adaptive $k$-Meteorologist algorithm is
$ O({n \choose K}/\epsilon_1^2 \log ({n \choose K}/\delta)) + {n \choose K} O\left(\frac{d^{K+1}}{\epsilon_1^2}\log({n \choose K}/\delta)\right)$, which can be simplified to
\begin{align}
O\left(\frac{n^K d^{K+1} K }{\epsilon_1^2} \log \frac{n}{\delta}\right) \label{eqn:akm}.
\end{align}

\subsection{LearnAndSelect}
\label{sub:episode}

\begin{lemma}
	\label{lem:episode}
    (Exploration Episode Lemma)
	The following holds w.p. $1-\delta_1$. At the end of each iteration of the while loop (line \ref{lin:while} -- line \ref{lin:endwhile}) in the LearnAndSelect algorithm (Algorithm \ref{alg:episode}), one of two things will happen: either the target $g$ or another feature-value vector that has not been visited $m$ times will be visited, or some feature will be eliminated as a possible parent for some other feature.
\end{lemma}

\begin{proof}

The idea behind the while loop is the Explore or Exploit lemma (Lemma \ref{lem:ee}) and the diameter assumption. The diameter assumption allows the algorithm to reach $g$ with high probability. The Explore or Exploit lemma allows the algorithm to either reach $g$ or end up gathering new data, therefore making progress towards the sub-algorithms and the Adaptive $k$-Meteorologist algorithm.

First, we compute how long $H$ needs to be in order for a good policy to reach $g$ with high probability. By the diameter assumption, there exists a policy expected to reach $g$ within $D$ steps, thereby obtaining a reward of $1$ from the artificially defined reward function $R$. By the Markov Inequality, the probability of reaching the goal within $2D$ steps is at least $\frac{1}{2}$. Thus the optimal average value within $2D$ steps is at least $\frac{1}{4D}$. If we used an $\epsilon$-optimal policy, it would have an expected value of at least $\frac{1}{4D} - \epsilon$. Let $\tau$ be trajectories of length $2D$. Then $\frac{1}{4D} - \epsilon = \frac{1}{2D}\sum_{\tau}Pr(escape)Pr(\tau | escape)TotalReward(\tau)$. The probability that the $\epsilon_1$-optimal policy reached the goal (escapes) can be lower bounded by the worst case scenario: every escape trajectory has every step giving a reward. That means the probability of reaching the goal (escape) is at least $\frac{1}{4D} - \epsilon$. Then probability of failing to reach the goal is at most $1 - \frac{1+ 2D\epsilon}{4D}$. Then by repeating this 2D-step trial $N$ times, the error probability is upper bounded by $\left(1 - \frac{1+ 2D\epsilon}{4D} \right)^N$. Then that means if we want to have a failure probability of $\delta_1$, we would need to repeat this $2D$-step sub-episode $\frac{\log(\delta_1)}{\log(1 - \frac{1+ 2D\epsilon}{4D})}$ times. We can simplify the denominator $\log(1 - \frac{1+ 2D\epsilon}{4D})$ by the upper bound $\log(1 - \frac{1}{4D})$. Note that the $\log$ function is concave, so we can upper bound it with its first order approximation around $\log(1)$ i.e. by $O(-\frac{1}{D})$. Simplifying the whole fraction becomes $O(D\log(1/\delta_1))$. We also want this to hold over every trial in which we can reach the target, meaning we want it to hold $m$ times. Thus we use the union bound and end up with
\begin{align}
O(D^2\log(m/\delta_1)) \label{eqn:trialsteps}
\end{align}
as the number of steps we need before reaching the goal with high probability. Thus this is a lower bound for $H$ and we also know in this case the while loop will terminate early after these many steps.

Now we consider the case when the probability of escaping is at least $\epsilon$. Then we need a much larger $H$ because we need the data from getting stuck to run the Superset test (line \ref{lin:ss}). This data is also used towards Adaptive $k$-Meteorologists to help with eliminating incorrect parent sets.

For the Superset test, we apply the Superset assumption (Assumption \ref{ass:superset}). We know that we got stuck so the data that we have is where the escape probability is high, thus meeting the superset assumption requirements of visiting distinguishing states (states where our model is incorrect). Because of the superset assumption (Assumption \ref{ass:superset}), we know that whenever we run the superset test, we will eliminate at least one parent set. Since there are a finite number of parent sets and a finite number of features, we cannot keep running the superset test forever. Eventually we will either eliminate enough features to make $g$ unreachable, or we will eliminate all incorrect parent sets, leaving a correct model that we can use to reach the target $g$. 

The value of $m$ comes from the sample complexity result of learning a single parent set from Adaptive $k$-Meteorologist, which is $O(\frac{d}{\epsilon_1^2}\log(1/\delta_1))$, and subbing in the Simulation lemma (Lemma \ref{lem:simulation}) with $\epsilon_1 = \alpha$. However we want this to hold over all parents and all values, so we also need to perform a union bound over ${n \choose K}d^K$, resulting in
\begin{align}
m=O\left(\frac{K^2 n^4|A|^4d^{4K+1} R_{\max}^4}{\epsilon^4 \max(D,T_\epsilon)^{-16}}\log(nd/\delta_1)\right) \label{eqn:m}
\end{align}
Note that the $\max(D,T_\epsilon)$ term is from using it in the simulation lemma (Lemma \ref{lem:simulation}), in order to have an $\epsilon$-optimal policy over that length.

Now we count how much data we need from getting stuck to perform the Superset Test. We need to gather new data for the prediction of supersets of size $2K$. From the discussion on the pigeonhole principle in the superset assumption, all we need is $d^{2K} m$ samples in order to accumulate enough data. However since we only have an escape probability of $\epsilon$, we need to add repeats to escape with high probability, just like we did earlier. This means we need an additional factor of $O(\frac{1}{\epsilon}\log(1/\delta_1))$. So we need 
\begin{align}
H = O\left(\frac{d^{6K+1}n^{4} K^2 |A|^4 R_{\max}^4}{\epsilon^5 \max(D,T_\epsilon)^{-18}}\log\left(\frac{nm}{\delta_1}\right)\right) \label{eqn:supersetdata}
\end{align}

\end{proof}

\begin{lemma}
	\label{lem:selection}
    (LearnAndSelect Lemma)
	The following holds w.p. $1-\delta_1$. After LearnAndSelect (Algorithm \ref{alg:episode}) is finished, all targets $g$ will either have been visited $m$ times, or one of its features will have been eliminated. If $K \geq J$, all necessary features will remain. This will take $O\left(\frac{K^2 n^{3K+4}|A|^4d^{7K+1} R_{\max}^4}{\epsilon^5 \max(D,T_\epsilon)^{-18} }\log(nm/\delta_1)\right)$ steps.
\end{lemma}

\begin{proof}

Now we will count how many steps LearnAndSelect (Algorithm \ref{alg:episode}) will take. Every while loop (line \ref{lin:while} -- line \ref{lin:endwhile}) contributes to one of two cases: visiting a feature-value vector that has not yet been visited $m$ times, or the Superset Test.

Since there are at most $O((dn)^{K})$ superset tests (each test eliminates at least one parent set), and we know how much data each superset test requires (equation \ref{eqn:supersetdata}), we combine those to get a total of
\begin{align}
H = O\left(\frac{d^{7K+1}n^{K+4} K^2 |A|^4 R_{\max}^4}{\epsilon^5 \max(D,T_\epsilon)^{-18}}\log\left(\frac{nm}{\delta_1}\right)\right)
\end{align}
steps towards superset tests.

Our targets are subsets of features and values of size $2K$, thus there are $O((dn)^{3K})$ targets. Each target needs to be visited $m$ times, thus $O((dn)^{3K}m)$ total steps are needed. Then the number of steps this all takes is
\begin{align}
O\left(\frac{K^2 n^{3K+4}|A|^4d^{7K+1} R_{\max}^4}{\epsilon^4 \max(D,T_\epsilon)^{-18}}\log(nm/\delta_1)\right)
\end{align}

The contributions to Adaptive $k$-Meteorologist are incidental, and are already counted as part of the superset tests. Thus, combining the number of steps that contribute to reaching targets and the number of steps for superset tests, we get
\begin{align}
O\left(\frac{K^2 n^{3K+4}|A|^4d^{7K+1} R_{\max}^4}{\epsilon^5 \max(D,T_\epsilon)^{-18}}\log(nm/\delta_1)\right) \label{eqn:exploration}
\end{align}
steps before LearnAndSelect finishes.
\end{proof}

\subsection{Instantiating PAC-FactoredMDP-RL}
\label{sub:exploit}

\begin{algorithm}[h]
	\begin{algorithmic}[1]
  \STATE \textbf{Input:} $K, F, m$
  \STATE Run and update $\pi$ along with the Adaptive $k$-Meteorologist Algorithm for $O\left(\frac{(dn)^{3K+1} K D^{18} n^4 T_{\epsilon}^9 |A|^4 R_{\max}^4}{\epsilon^5} \log \frac{nm}{\delta}\right)$  steps
  \begin{itemize}
  \item $\pi$ = optimistic policy
  \end{itemize}
  \end{algorithmic}
 \caption{PAC-FactoredMDP-RL}
 \label{alg:rl_given_subset}
\end{algorithm}

\begin{lemma}
	\label{lem:rlwithselected}
    (PAC-FactoredMDP-RL Lemma)
	Suppose $K \geq J$. Then we instantiate PAC-FactoredMDP-RL with a specific algorithm (Algorithm \ref{alg:rl_given_subset}). Then it will execute for $O\left(\frac{K^2  n^{3K+4}|A|^4d^{7K+1} R_{\max}^4}{\epsilon^5 \max(D,T_\epsilon)^{-18}} \log(nm/\delta_1)\right)$ steps. The average reward over LearnAndSelect (Algorithm \ref{alg:episode}) and PAC-FactoredMDP-RL combined is $\epsilon$-optimal, with probability $1-\delta$.
\end{lemma}

\begin{proof}

Suppose $K$ is at least as large as the in-degree of the necessary features. By the LearnAndSelect lemma (Lemma \ref{lem:selection}), the necessary features still remain and have not been eliminated. Also, we will have at least $m$ samples from the true parents of all necessary features (and rewards), and thus the true parents of all the necessary features still remain as candidates models that have not yet been eliminated by Adaptive $k$-Meteorologists.

Then we keep executing an optimistic, optimal policy. Since after LearnAndSelect we have accrued data from all remaining parent sets as well as all pairs of parent sets, no sub-algorithm will make a null prediction. Furthermore, Adaptive $k$-Meteorologist relies on gathering data from pairs of conflicting parent sets when those parent sets give different predictions. Since we have also targeted those pairs of parent sets, we know those pairs are all reachable with our remaining candidate models. LearnAndSelect has already eliminated all the features that were unreachable, so we know we will be able to get the correct distinguishing data we need. Thus during RL with Selected Features we will plan by picking the most optimistic model possible out of all remaining candidate parent sets.

As long as we keep visiting states where Adaptive $k$-Meteorologists has reached consensus and do not escape (to states without consensus), then the average reward will be $\epsilon$-close to the predicted average reward. Using the Necessary Feature lemma (Lemma \ref{lem:necessary}), the transition dynamics of the unnecessary features do not matter, and since our model is optimistic, the average reward over those steps will be $\epsilon$-close to the optimal policy. 

From the Explore or Exploit lemma (Lemma \ref{lem:ee}) if we are not $\epsilon$-optimal, then it means we have at least an $\epsilon$ probability of escaping. Using the same reasoning about repeating trials for high probability in Lemma \ref{lem:episode}, we need $O(\frac{1}{\epsilon}\log(1/\delta))$ repeats to escape with high probability. If we reach states where consensus fails, then we accrue a new sample towards eliminating an incorrect parent set. From Adaptive $k$-Meteorologist (equation \ref{eqn:akm}), the total sample complexity is $O\left(\frac{n^K d^{K+1} K }{\epsilon_1^2} \log \frac{n}{\delta}\right)$. We use the Simulation lemma (Lemma \ref{lem:simulation}) to relate model error with policy error, resulting in $O\left(\frac{n^K d^{5K+1} K n^4 T_{\epsilon}^9 |A|^4 R_{\max}^4}{\epsilon^5} \log \frac{n}{\delta}\right)$; furthermore, an additional factor of $T^{\epsilon}/\epsilon$ was also added because we need a stretch of $T_{\epsilon}$ consecutive steps with consensus in order to attain near-optimal average reward (i.e. a failed consensus can affect a stretch of $T_{\epsilon}$ steps), and to escape with high probability.

In order to offset the sample complexity count just computed, we need $O(1/\epsilon)$ times as many $\epsilon$-optimal steps to average $O(\epsilon)$ error. We also want to offset the steps from LearnAndSelect. Note that the number of steps from LearnAndSelect dominates (equation \ref{eqn:exploration}). Thus the total number of steps required is
\begin{align}
O\left(\frac{K^2  n^{3K+4}|A|^4d^{7K+1} R_{\max}^4}{\epsilon^6 \max(D,T_\epsilon)^{-18}} \log(nm/\delta_1)\right) \label{eqn:exploitation}
\end{align}
\end{proof}

%% file: toggle.tex
\section{Toggle Domain}

Toggle has one binary necessary feature $f_1$ and potentially two binary unnecessary features $f_2,f_3$. There are two actions, $a_1,a_2$. The parent set of $f_1$ is $(f_1)$, and the parent sets of $f_2,f_3$ are both $(f_2,f_3)$; this means the in-degree of necessary features is $1$ and the in-degree of the whole domain is $2$.

For states where $f_1 = 0$, action $a_1$ is optimal and gives a stochastic reward of 1 with probability $0.05$ whereas action $a_2$ gives a deterministic reward of $0.025$ and is meant to make learning difficult. Action $a_1$ also deterministically transitions to $f_1=1$ and action $a_2$ deterministically stays the same. When $f_1=1$, the dynamics of the actions are switched, so the optimal policy is to keep alternating the corresponding optimal action. For the unnecessary features, no actions give any reward. When $(f_2,f_3)=(0,0)$, action $a_1$ has an independent probability of $0.1$ of setting either feature value to 1, and otherwise $a_1$ sets them both back to 0. When $(f_2,f_3)=(1,1)$, action $a_2$ deterministically keeps both values at 1, and otherwise sets them back to 0. These dynamics make it clear that the in-degree of the unnecessary features is 2.

We constructed Toggle we need $m=100$ to learn accurate dynamics. Furthermore, the two unnecessary features $f_2,f_3$ satisfy the Superset Assumption.